\documentclass{article}

\usepackage{microtype}
\usepackage{graphicx}
\usepackage{subfigure}
\usepackage{booktabs} % for professional tables
\usepackage{enumitem}
\usepackage{hyperref}

\usepackage[accepted]{icml2023}

\usepackage{amsmath}
\usepackage{amssymb}
\usepackage{mathtools}
\usepackage{amsthm}

\usepackage[capitalize,noabbrev]{cleveref}

\theoremstyle{plain}
\newtheorem{theorem}{Theorem}[section]
\newtheorem{proposition}[theorem]{Proposition}
\newtheorem{lemma}[theorem]{Lemma}

\theoremstyle{definition}
\newtheorem{definition}[theorem]{Definition}

\theoremstyle{remark}

\usepackage{dsfont}
\usepackage{amsthm}

\newcommand{\vertexSet}[0]{V}

\newcommand{\prob}[0]{P}

\newcommand{\expect}[0]{\mathds{E}}
\newcommand{\rand}[1]{\mathbf{#1}}

\newcommand{\embdingSpace}[0]{\mathcal{Z}}

\newcommand{\kernelSpace}[0]{\mathcal{F}}
\newcommand{\Hnorm}[1]{\|#1\|_{\kernelSpace}}
\newcommand{\kernel}[0]{\mathcal{K}}

\newcommand{\loss}[0]{\mathcal{L}}
\newcommand{\gnnModel}[0]{g}

\newcommand{\graphStruct}[0]{\mathcal{G}}
\newcommand{\neighbor}[1]{\mathrm{N}_{#1}}

\newcommand{\ego}[1]{G_{#1}}

\newcommand{\task}[0]{\mathcal{T}}

\newcommand{\cataForget}[0]{\textbf{CFR}}
\newcommand{\hypothesis}[0]{\mathcal{H}}

\newcommand{\jointDistCondVertSet}[2]{\prob(\mathbf{y}_{#1},\mathbf{G}_{#1}|#2)}
\newcommand{\jointDistCondVerGraph}[2]{\prob(\mathbf{y}, \mathbf{G}_{v}|\vertexSet_{#1}, \graphStruct_{\task_{#2}})}
\newcommand{\jointDistCondVerGraphInduced}[2]{\prob(\mathbf{y}, g(\mathbf{G}_{v})|\vertexSet_{#1}, \graphStruct_{\task_{#2}})}

\usepackage{amsmath,amsfonts,bm}

\def\eqref#1{equation~\ref{#1}}
\def\1{\bm{1}}

\DeclareMathAlphabet{\mathsfit}{\encodingdefault}{\sfdefault}{m}{sl}
\SetMathAlphabet{\mathsfit}{bold}{\encodingdefault}{\sfdefault}{bx}{n}

\newcommand{\R}{\mathbb{R}}

\DeclareMathOperator*{\argmin}{arg\,min}

\usepackage[textsize=tiny]{todonotes}

\begin{document}

\twocolumn[
\icmltitle{ Towards Robust Graph Incremental Learning on Evolving Graphs}

% It is OKAY to include author information, even for blind
% submissions: the style file will automatically remove it for you
% unless you've provided the [accepted] option to the icml2023
% package.

% List of affiliations: The first argument should be a (short)
% identifier you will use later to specify author affiliations
% Academic affiliations should list Department, University, City, Region, Country
% Industry affiliations should list Company, City, Region, Country

% You can specify symbols, otherwise they are numbered in order.
% Ideally, you should not use this facility. Affiliations will be numbered
% in order of appearance and this is the preferred way.
\icmlsetsymbol{equal}{*}

\begin{icmlauthorlist}
\icmlauthor{Junwei Su}{yyy}
\icmlauthor{Difan Zou}{yyy}
\icmlauthor{Zijun Zhang}{yyy2}
\icmlauthor{Chuan Wu}{yyy}
% \icmlauthor{Firstname5 Lastname5}{yyy}
% \icmlauthor{Firstname6 Lastname6}{sch,yyy,comp}
% \icmlauthor{Firstname7 Lastname7}{comp}
% %\icmlauthor{}{sch}
% \icmlauthor{Firstname8 Lastname8}{sch}
% \icmlauthor{Firstname8 Lastname8}{yyy,comp}
%\icmlauthor{}{sch}
% %\icmlauthor{}{sch}
\end{icmlauthorlist}

\icmlaffiliation{yyy}{Department of Computer Science, University of Hong Kong}
\icmlaffiliation{yyy2}{Department of Computer Science, University of Wu Han}

\icmlcorrespondingauthor{Junwei Su}{junweisu@connect.hku.hk}
% \icmlcorrespondingauthor{Firstname2 Lastname2}{first2.last2@www.uk}

% You may provide any keywords that you
% find helpful for describing your paper; these are used to populate
% the "keywords" metadata in the PDF but will not be shown in the document
\icmlkeywords{Machine Learning, ICML}

\vskip 0.3in
]

% this must go after the closing bracket ] following \twocolumn[ ...

% This command actually creates the footnote in the first column
% listing the affiliations and the copyright notice.
% The command takes one argument, which is text to display at the start of the footnote.
% The \icmlEqualContribution command is standard text for equal contribution.
% Remove it (just {}) if you do not need this facility.

\printAffiliationsAndNotice{}  % leave blank if no need to mention equal contribution
% \printAffiliationsAndNotice{\icmlEqualContribution} % otherwise use the standard text.

\begin{abstract}
Incremental learning is a machine learning approach that involves training a model on a sequence of tasks, rather than all tasks at once. This ability to learn incrementally from a stream of tasks is crucial for many real-world applications. However, incremental learning is a challenging problem on graph-structured data, as many graph-related problems involve prediction tasks for each individual node, known as Node-wise Graph Incremental Learning (NGIL). This introduces non-independent and non-identically distributed characteristics in the sample data generation process, making it difficult to maintain the performance of the model as new tasks are added. In this paper, we focus on the inductive NGIL problem, which accounts for the evolution of graph structure (structural shift) induced by emerging tasks. We provide a formal formulation and analysis of the problem, and propose a novel regularization-based technique called Structural-Shift-Risk-Mitigation (SSRM) to mitigate the impact of the structural shift on catastrophic forgetting of the inductive NGIL problem. We show that the structural shift can lead to a shift in the input distribution for the existing tasks, and further lead to an increased risk of catastrophic forgetting. Through comprehensive empirical studies with several benchmark datasets, we demonstrate that our proposed method, Structural-Shift-Risk-Mitigation (SSRM), is flexible and easy to adapt to improve the performance of state-of-the-art GNN incremental learning frameworks in the inductive setting. \let\thefootnote\relax\footnote{Implementation available at: \url{https://github.com/littleTown93/NGIL_Evolve}}
\end{abstract}

\section{Introduction}\label{sec:introduction}
 Humans are capable of acquiring new information continuously while retaining previously obtained knowledge.  This seemingly natural capability, however, is difficult but important for deep neural networks (DNNs) to acquire~\cite{wu2021incremental}.  Incremental learning, also known as continual learning or life-long learning, studies machine learning approaches that allow a model to continuously acquire new knowledge while retaining previously obtained knowledge. This is important because it allows the model to adapt to new information without forgetting past knowledge. In the general formulation of incremental learning, a stream of tasks arrives sequentially and the model goes through rounds of training sessions to accumulate knowledge for a particular objective (such as classification). The goal is to find a learning algorithm that can incrementally update the model's parameters based on the new task without suffering from {\em catastrophic forgetting}~\cite{cl_survey}, which refers to the inability to retain previously learned information when learning new tasks. Incremental learning is crucial for the practicality of machine learning systems as it allows the model to adapt to new information without the need for frequent retraining, which can be costly.

While incremental learning has been extensively studied for Euclidean data (such as images and text)~\cite{cl_nlp,cl_application}, there has been relatively little research on incremental learning for graph-structured data~\cite{zhang2022cglb}. Graphs are often generated continuously in real-life scenarios, making them an ideal application for incremental learning. For example, in a citation network, new papers and their associated citations may emerge, and a document classifier needs to continuously adapt its parameters to distinguish the documents of newly emerged research fields~\cite{zhou2021overcoming}. In social networks, the distributions of users' friendships and activities depend on when and where the networks are collected~\cite{gama2014survey}. In financial networks, the payment flows between transactions and the appearance of illicit transactions have strong correlations with external contextual factors such as time and market~\cite{zliobaite2010learning}. Therefore, it is important to develop Graph Incremental Learning (GIL) methods that can handle new tasks over newly emerged graph data while maintaining model performance, particularly for highly dynamic systems such as citation networks, online social networks, and financial networks. Graph neural networks (GNNs) are popular and effective tools for modelling graph and relational data structures~\cite{gnn_survey}, while the ability to incrementally learn from data streams is %a challenge 
needed for building practical %artificial intelligence 
GNN-driven systems.

One of the major challenges in incremental learning on graph-structured data lies in the prediction of individual nodes, known as Node-wise GIL (NGIL). This type of problem introduces non-independent and non-identically distributed characteristics in the sample data generation process, making it difficult to maintain the performance of the model as new tasks are added. Existing research on NGIL has primarily focused on a {\em transductive} setting, where the graph structures among tasks are assumed to be independent~\cite{zhou2021overcoming,liu2021overcoming}. However, in many real-life scenarios, such as the examples %previously mentioned
above, it is inevitable that the emerging graph associated with the new task would expand on the existing graph, thereby altering the structural information of the existing vertices, %referred to as
i.e., structural shift. As the structural information of the graph is a crucial input for GNNs, an evolving graph structure can greatly affect the model's prediction and generalization abilities. This is referred to as the {\em inductive} NGIL problem. Fig.~\ref{fig:diff_ngil} provides a graphical illustration of the difference between transductive and inductive NGIL. To advance the practicality of the NGIL framework, it is important to study this inductive setting. However, currently, there is a lack of a clear problem formulation and theoretical understanding of the problem, making it challenging to develop effective solutions for inductive NGIL. 

%contribution and what we do
In this paper, we delve into the complex inductive NGIL problem. We carefully formulate and rigorously analyze this problem, and propose a novel regularization-based technique to effectively mitigate the impact of the structural shift on catastrophic forgetting (as measured by retention of model performance in previous tasks) of the inductive NGIL problem. %More concretely, 
Our contributions are summarized as follows.

\begin{itemize}[ leftmargin=*]
   \item We present a mathematical formulation of the inductive NGIL problem that quantifies the effect of changes in the graph structure on the model's ability to generalize to previously seen tasks/vertices. This formulation provides a clear and detailed understanding of the problem and is essential for designing and evaluating effective solutions.

\item We conduct a formal analysis of the impact of the structural shift on catastrophic forgetting of the inductive NGIL problem. We show that structural shift can lead to a shift in the input distribution for existing tasks (Proposition~\ref{prop:imbalanced_observation}), and we derive a bound on catastrophic forgetting that indicates that the risk of forgetting is positively related to the extent of structural shift present in the inductive setting (Theorem~\ref{thm:forget_bound}). This sheds insight into developing algorithms for inductive NGIL.

\item Based on the analysis, we propose a novel regularization method utilizing the divergence minimization principle, which encourages the model to converge to a latent space where the difference of the distribution induced by the structural shift is minimised. We show that such an approach can reduce the risk of catastrophic forgetting (Thereom~\ref{thm:induced_cfr_bound}), and we term our proposed method as Structural-Shift-Risk-Mitigation (SSRM). SSRM is easy to implement and can be easily adapted to existing GNN incremental learning frameworks to boost their performance in the inductive setting.

\item We validate the predicted effect of structural shift on catastrophic forgetting and the effectiveness of our proposed SSRM method through comprehensive empirical studies. Our results show that SSRM can consistently improve the performance of state-of-the-art incremental learning frameworks on the inductive NGIL problem.
\end{itemize}

\begin{figure}[!t]
\centering
\includegraphics[width=0.5\textwidth]{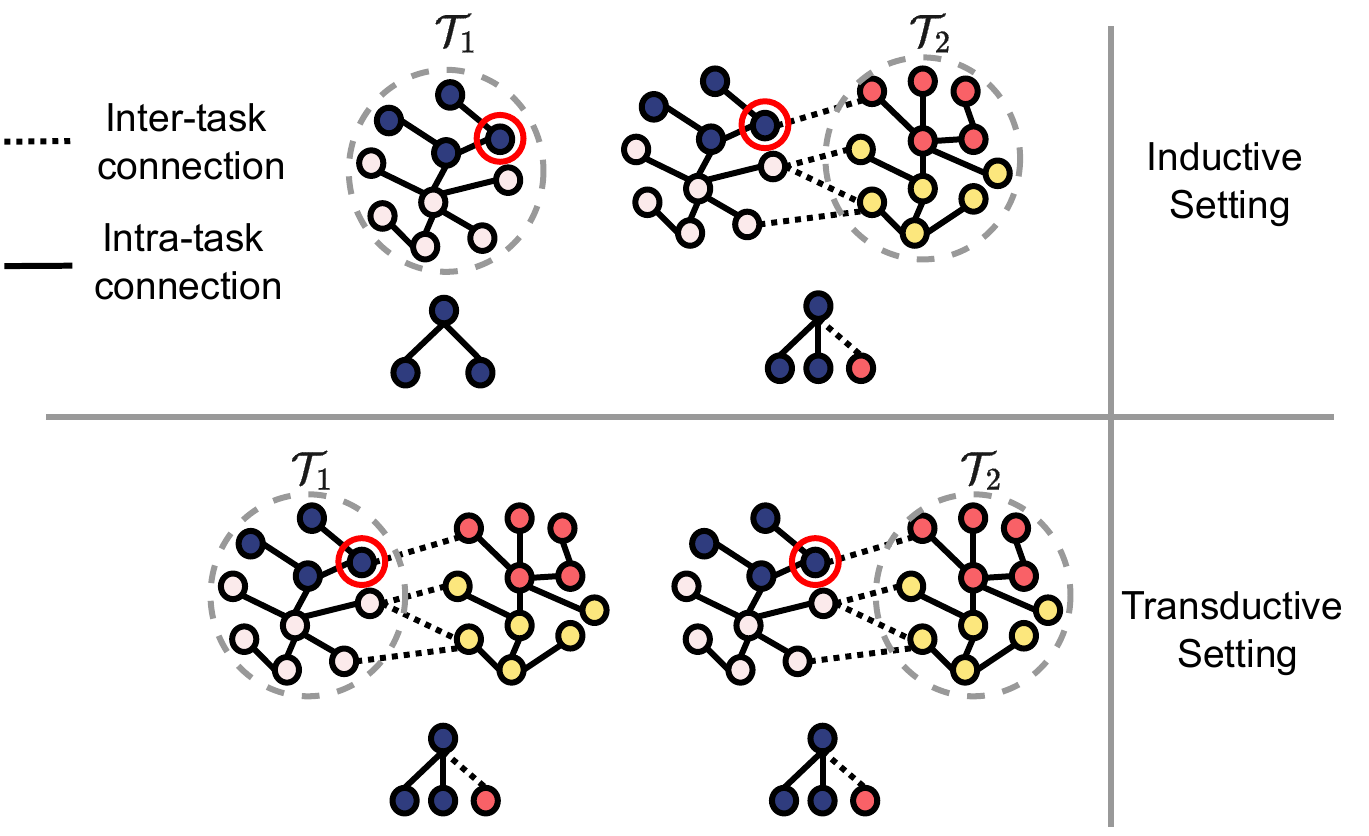}
\vspace{-3mm}
\caption{An illustration of the difference between transductive and inductive NGIL. $\task_1$ and $\task_2$ are two consecutive tasks. In the transductive setting, the 1-hop ego graph of the vertex with a red circle would remain the same, while in the inductive setting, the graph may change as new tasks are introduced and the overall graph structure evolves.}
\label{fig:diff_ngil}
\vspace{-0.3cm}
\end{figure}

\section{Related Work}\label{sec:related_work}
\subsection{Incremental Learning}
Incremental learning, also known as continual or lifelong learning, has gained increasing attention in recent years and has been extensively explored on Euclidean data. We refer readers to the surveys~\cite{cl_survey,cl_survey2,cl_nlp} for a more comprehensive review of these works. The primary challenge of incremental learning is to address the catastrophic forgetting problem, which refers to the drastic degradation in a model's performance on previous tasks after being trained on new tasks.

Existing approaches for addressing this problem can be broadly categorized into three types: regularization-based methods, experience-replay-based methods, and parameter-isolation-based methods. Regularization-based methods aim to maintain the model's performance on previous tasks by penalizing large changes in the model parameters~\cite{jung2016less,li2017learning,kirkpatrick2017overcoming,farajtabar2020orthogonal,saha2021gradient}. Parameter-isolation-based methods prevent drastic changes to the parameters that are important for previous tasks by continually introducing new parameters for new tasks ~\cite{rusu2016progressive,yoon2017lifelong,yoon2019scalable,wortsman2020supermasks,wu2019large}. Experience-replay-based methods select a set of representative data from previous tasks, which are used to retrain the model with the new task data to prevent forgetting ~\cite{lopez2017gradient,shin2017continual,aljundi2019gradient,caccia2020online,chrysakis2020online,knoblauch2020optimal}.

Our proposed method, SSRM, is a novel regularization-based technique that addresses the unique challenge of the structural shift in the inductive NGIL. Unlike existing regularization methods, which focus on minimizing the effect of updates from new tasks, SSRM aims to minimize the impact of the structural shift on the generalization of the model on the existing tasks by finding a latent space where the impact of the structural shift is minimized. This is achieved by minimizing the distance between the representations of vertices in the previous and current graph structures through the inclusion of a structural shift mitigation term in the learning objective. It is important to note that SSRM should not be used as a standalone method to overcome catastrophic forgetting but should be used as an additional regularizer to improve performance when there is a structural shift.

\subsection{Graph Incremental Learning}
Recently, there has been a surge of interest in GIL due to its practical significance in various applications~\cite{wang2022lifelong,xu2020graphsail,daruna2021continual,kou2020disentangle,ahrabian2021structure,cai2022multimodal,wang2020bridging,liu2021overcoming,zhang2021hierarchical,zhou2021overcoming,carta2021catastrophic,zhang2022cglb,kim2022dygrain,tan2022graph}. However, most existing works in NGIL focus on a transductive setting, where the entire graph structure is known before performing the task or where the inter-connection between different tasks is ignored. In contrast, the inductive NGIL problem, where the graph structure evolves as new tasks are introduced, is less explored and lacks a clear problem formulation and theoretical understanding. There is currently a gap in understanding the inductive NGIL problem, which our work aims to address. In this paper, we highlight the importance of addressing the structural shift and propose a novel regularization-based technique, SSRM, to mitigate the impact of the structural shift on the inductive NGIL problem. Our work lays down a solid foundation for future research in this area.

Finally, it is important to note that there is another area of research known as dynamic graph learning~\cite{galke2021lifelong,wang2020streaming,han2020graph,yu2018netwalk,nguyen2018continuous,ma2020streaming,feng2020incremental,bielak2022fildne}, which focuses on enabling GNNs to capture the changing graph structures. The goal of dynamic graph learning is to capture the temporal dynamics of the graph into the representation vectors, while having access to all previous information. In contrast, GIL addresses the problem of catastrophic forgetting, in which the model's performance on previous tasks degrades after learning new tasks. For evaluation, a dynamic graph learning algorithm is only tested on the latest data, while GIL models are also evaluated on past data. Therefore, dynamic graph learning and GIL are two independent research directions with different focuses and should be considered separately.

\section{Preliminary and Problem Formulation}\label{sec:problem_formulation}
In this section, we present the preliminary and formulation of the inductive NGIL problem. We use bold letters to denote random variables, while the corresponding realizations are represented with thin letters.

We assume the existence of a stream of training tasks $\task_1, \task_2,..., \task_m$, characterized by observed vertex batches $\vertexSet_1,\vertexSet_2,...,\vertexSet_m$ that are drawn from an unknown undirected graph $\graphStruct$. Each vertex $v$ is associated with a node feature $x_v$ and a target label $y_v$. The observed graph structure at training task $\task_i$ is induced by the accumulative vertices and given by $\graphStruct_{\task_i} = \graphStruct[\bigcup_{j=1}^i \vertexSet_j]$. In this setting, the graph structure is evolving as the learning progresses through different training tasks. Fig.~\ref{fig:ngil} provides a graphical illustration.  

To accommodate the nature of node-level learning tasks, in which the information used for inference is aggregated within the k-hop neighborhood of a node, we adopt a local view in the learning problem formulation. We define $\neighbor{k}(v)$ as the k-hop neighborhood of vertex $v$, and the nodes in $\neighbor{k}(v)$ form an ego-graph $\ego{v}$, which consists of a (local) node feature matrix $X_v = \{x_u |u \in \neighbor{k}(v)\}$ and a (local) adjacency matrix $A_v = \{a_{uw}|u,w \in \neighbor{k}(v)\}$. We use $\rand{\ego{v}}$ as a random variable of the ego-graph for the target vertex $v$, whose realization is $\ego{v} = (A_v,X_v)$.  %Furthermore, we denote 
Let $\ego{\vertexSet} = \{\ego{v} |v \in \vertexSet \}$ denote the set of ego graphs associated with vertex set $\vertexSet$. Under this formulation, the ego-graph $\ego{v}$ can be viewed as the Markov blanket (containing all necessary information for the prediction problem) for the root vertex $v$. Therefore, we can see the prediction problem associated with data $\{(\ego{v},y_v)\}_{v \in \vertexSet_i}$ from training session $\task_i$ as drawn from an empirical joint distribution $\prob(\mathbf{y_{v}},\mathbf{G}_{v}|\vertexSet_i)$.

Let $\hypothesis$ denote the hypothesis space and $f \in \hypothesis$ be a classifier with $\hat{y}_v = f (G_v)$ and $\loss(.,.) \mapsto \R$ be a given loss function. We use $R_{\jointDistCondVertSet{v}{\vertexSet}}^{\loss}(f)$ to denote the generalization risk of the classifier $f$ with respect to $\loss$ and $\jointDistCondVertSet{v}{\vertexSet}$, and it is defined as follows:
 \begin{equation}\label{eq:gen_risk}
     R_{\jointDistCondVertSet{v}{\vertexSet}}^{\loss}(f) = \expect_{\jointDistCondVertSet{v}{\vertexSet}}[\loss(f(\ego{v}),y_v)].
 \end{equation}

{\bf Catastrophic Forgetting Risk.}
With the formulation above, the catastrophic forgetting risk ($\cataForget$) of a classifier $f$ after being trained on $\task_m$ can be characterized by the retention of performance on previous vertices, given by: 
\begin{equation}\label{eq:cf}
\begin{split}
      \cataForget(f) & := R_{\jointDistCondVertSet{v}{\vertexSet_1,...,\vertexSet_{m-1}}}^{\loss}(f)
\end{split}
\end{equation}
which translates to the retention of performance of the classifier $f$ from $\task_m$ on the previous tasks $\task_1,...,\task_{m-1}$.

\begin{figure*}[!t]
\centering
\includegraphics[width=0.85\textwidth]{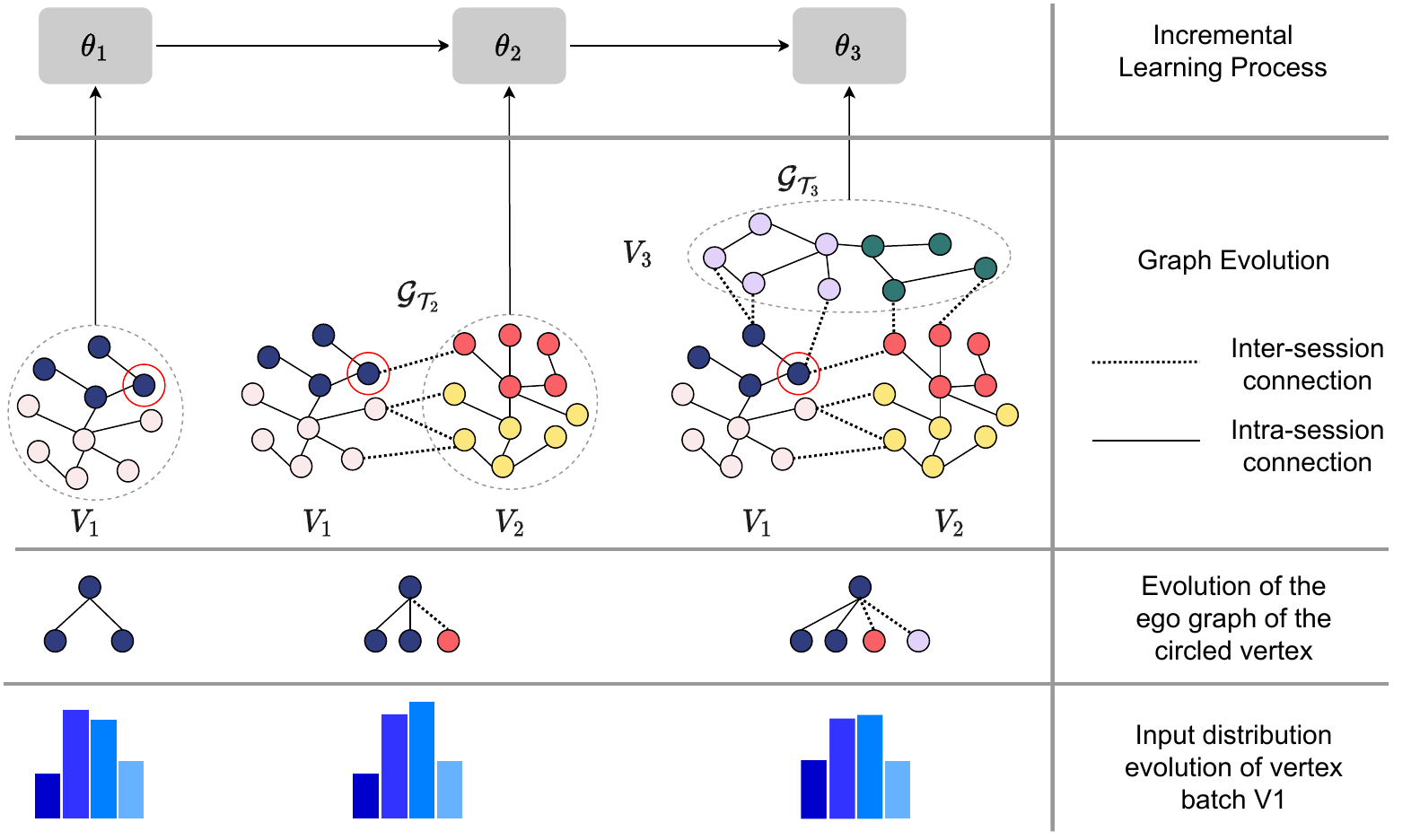}
\vspace{-2mm}
\caption{Illustration of the Progression in Inductive NGIL. Each task $i$ results in an update to the model's parameters from $\theta_{i-1}$ to $\theta_{i}$ using data from the new task. As new vertex batches associated with each task are introduced, the graph structure changes, potentially altering the input distribution of existing vertices through changes in their ego graphs.}
\label{fig:ngil}
\vspace{-3mm}
\end{figure*}

\section{Structural Shift and Catastrophic Forgetting}\label{sec:learnability}
In this section, we present our main results on the connection between the structural shift of the evolving graph structure and the catastrophic forgetting risk of the NGIL problem. We start by showing how the evolving graph structure can lead to a distributional shift in the underlying learning problem. We then derive a bound for catastrophic forgetting risk with respect to structural shift.

\subsection{Structural Shift}\label{subsec:structural_shift}
To illustrate the phenomenon of structural shift more concretely in NGIL, we consider a graph generation process with two communities (i.e., types of vertices): community 1 and community 2. The feature vectors $x_u$ of vertices $u$ in community 1 are drawn from a distribution $N_1$, and those of vertices in community 2 are drawn from another distribution $N_2$, where $\expect_{N_1}[x_u] \neq \expect_{N_2}[x_u]$. The connectivity probability for vertices of the same community is denoted as $p_{in}$, and the connectivity probability for vertices of different communities is denoted as $p_{out}$. Note the setting above is an instance of the commonly used graph generation model for node classification tasks, i.e., Contextual Stochastic Block Model (CSBM)~\cite{csbm}.

Consider an incremental learning setting consisting of two training tasks $\task_1, \task_2$ and associated observed vertex batches $\vertexSet_1,\vertexSet_2$. We define $C_{1}(\vertexSet)\mapsto \mathds{N}$ to be a function that counts the number of vertices in community 1 for a given vertex batch, and function $C_{2}(\vertexSet) \mapsto \mathds{N}$ is defined similarly. Consider a mean aggregation function that averages the node features of the 1-hop neighbors, i.e., $\text{mean-agg}(v) = \frac{1}{|\neighbor{1}(v)|} \sum_{u \in \neighbor{1}(v)} x_u $.

\begin{proposition}[Imbalanced Observation]\label{prop:imbalanced_observation}
    If  $\frac{C_1(V_1)}{C_2(V_1)} \neq \frac{C_1(V_2)}{C_2(V_2)},$
    then we have that 
    $\expect[\text{mean-agg}(v)|\graphStruct_{\task_1}] \neq \expect[\text{mean-agg}(v)|\graphStruct_{\task_2}],  \forall v \in V_1.$
\end{proposition}

The proof of Proposition~\ref{prop:imbalanced_observation} can be found in Appendix~\ref{appendix:prop_proof}. Proposition~\ref{prop:imbalanced_observation} indicates that if the ratio between the number of vertices from the two communities is different, then the appearance of $\task_2$ would alter the expected input of vertices from $\task_1$ in the new graph. Proposition~\ref{prop:imbalanced_observation} highlights the fact that while the underlying mechanism for generating features and connectivity may remain the same across tasks, imbalanced observations can still cause a shift in the input distribution of ego graphs. Furthermore, as the graph evolves, changes in observed properties such as node features and connectivity can also alter the properties of ego graphs of existing vertices, resulting in distributional differences between tasks. This dependency between the input distribution (ego graphs) of vertices $\vertexSet_i$ and the observed graph structure $\graphStruct_{\task_j}$ of the training session $\task_j$, i.e. $ \prob(\mathbf{G}_{v}|\vertexSet_i, \graphStruct_{\task_j})$, poses a unique risk of catastrophic forgetting for the model. Not only must the model retain information for existing data, but it must also continually adapt to changing input distributions of existing data as the graph evolves. 

In this paper, we focus on the effect of structural shift on catastrophic forgetting and assume that the labelling rule is the same for vertices in different sessions, i.e., $\prob(\mathbf{y}|\mathbf{G}_{v}, \task_i) = \prob(\mathbf{y}|\mathbf{G}_{v}, \task_j)$, $\forall i,j$. In addition, we gauge our analysis toward the case of two training tasks, referred to as the NGIL-2 problem. The NGIL-2 problem is characterized by three distributions: $\prob(\mathbf{y}, \mathbf{G}_{v}|\vertexSet_1, \graphStruct_{\task_1})$, $\prob(\mathbf{y}, \mathbf{G}_{v}|\vertexSet_1, \graphStruct_{\task_2})$, and $\prob(\mathbf{y}, \mathbf{G}_{v}|\vertexSet_2, \graphStruct_{\task_2})$, which are the distributions of $\vertexSet_1$ in graphs $\graphStruct_{\task_1}$ and $\graphStruct_{\task_2}$, and the empirical distribution of $\vertexSet_2$ in graph $\graphStruct_{\task_2}$. The analysis can be easily extended to multiple training tasks by recursively applying the analysis and treating consecutive tasks as a joint task.

\subsection{Structural Shift on Catastrophic Forgetting Risk}

\begin{definition}[Maximum Mean Discrepancy]\label{def:mmd}
Let $\kernelSpace$ be a reproducing kernel Hilbert space (RKHS) with kernel $\kernel$
and norm $\Hnorm{.}$. Then the Maximum Mean Discrepancy (MMD) between distribution $\prob_1$ and $\prob_2$ is defined as:
$$d_{\mathrm{MMD}} (\prob_1,\prob_2) := \sup_{f \in \kernelSpace: \Hnorm{f} \leq 1} \expect_{\prob_1}[f(x)] - \expect_{\prob_2}[f(x)]$$
\end{definition}
MMD is a distance metric between distributions that leverages kernel embedding and has been commonly used to quantify the difference between two distributions~\cite{gretton2012kernel}. MMD admits efficient estimation, as well as fast convergence properties, which are of chief importance in our analysis and the proposed method.

In this paper, we use MMD distance to characterize the effect of the emerging vertex batch on the input distribution of the previous vertex batch. As discussed in subsection~\ref{subsec:structural_shift}, if the structural shift between the vertex batches is large, the appearance of the new vertex batch will significantly change the input distribution of the previous vertex batch, leading to a larger MMD distance. Next, we use MMD to formalize the relation between catastrophic forgetting risk and structural shift.

 \begin{theorem}[CFR Bound]\label{thm:forget_bound}
 Let $\hypothesis = \{ f \in \kernelSpace_{\kernel} : \Hnorm{f} \leq 1\}$ where $\kernelSpace_{\kernel} $ is a RKHS with its associated kernel $\kernel$.
Let $ \jointDistCondVerGraph{1}{1}$, $ \jointDistCondVerGraph{1}{2}$ and $ \jointDistCondVerGraph{2}{2}$ be the three distributions that characterize a NGIL-2 problem. Then for a given loss function $\loss^q(a,b)$ of the form $|a - b|^q$ , for every $h \in \mathcal{H}$, we have
\begin{equation*}
\begin{split}
        \cataForget(h) & \leq R_{\jointDistCondVerGraph{2}{2}}^{\loss^q}(h)  \\
    &+ \underbrace{ 2 * d_{\mathrm{MMD}}(\prob(\rand{G}_{v}|\vertexSet_1, 
  \graphStruct_{\task_1}), \prob(\rand{G}_{v}|\vertexSet_1, 
 \graphStruct_{\task_2}))}_{\text{Structural Shift}} \\
  &+ \underbrace{  d_{\mathrm{MMD}}(\prob(\rand{G}_{v}|\vertexSet_1, 
  \graphStruct_{\task_1}), \prob(\rand{G}_{v}|\vertexSet_2, 
 \graphStruct_{\task_2}))}_{\text{Structural Shift}}
    + \lambda,
\end{split}
\end{equation*}
where $$ \lambda = \min_{h \in \mathcal{H}}   R_{\jointDistCondVerGraph{1}{2}}^{\loss^q}(h)+R_{\jointDistCondVerGraph{2}{2}}^{\loss^q}(h).$$
\end{theorem}

The proof of Theorem~\ref{thm:forget_bound} can be found in Appendix~\ref{appendix:theorem1_proof}. Theorem~\ref{thm:forget_bound} provides a formal bound on the catastrophic forgetting risk (as measured by the retention of performance on the previous task) in inductive NGIL. The first term $R_{\jointDistCondVerGraph{2}{2}}^{\loss^q}(h)$ represents the model's performance on the newest task. The second term, $d_{\mathrm{MMD}}(\prob(\rand{G}_{v}|\vertexSet_1, 
  \graphStruct_{\task_1}), \prob(\rand{G}_{v}|\vertexSet_1, 
 \graphStruct_{\task_2}))$, captures the structural shift induced by the emergence of the new task, while the third term, $d_{\mathrm{MMD}}(\prob(\rand{G}_{v}|\vertexSet_1, 
  \graphStruct_{\task_1}), \prob(\rand{G}_{v}|\vertexSet_2, 
 \graphStruct_{\task_2}))$, captures the structural shift induced by the difference between the new and old tasks.  This theorem formalizes the effect of structural shift on the catastrophic forgetting risk in NGIL and can be used to analyze and develop NGIL frameworks.

\section{ Structural Shift Risk Mitigation}\label{sec:methodology}
As shown in the previous section, structural shift plays a significant role in causing catastrophic forgetting risks in NGIL problems. In order to mitigate this issue, we propose a method based on divergence minimization. This involves encouraging the model to converge to a latent space that is invariant to the distributional shift caused by the evolving graph structure. To achieve this, we modify the learning objective to minimize the distance between the representation of the vertices in the previous graph structure and the new graph structure. This is done through the inclusion of terms that capture the distance between the representations of vertices in different sessions and %serves to
regularize the model.

Here, we consider a GNN as the backbone for incremental learning, where the GNN serves as an embedding function that combines the graph structure and node features to learn a representation vector for each node. This representation is then passed to a task-specific decoder function for prediction. Formally, let $\embdingSpace$ be the embedding space for the learned representation of vertices, and $\hypothesis_{g}$ represent the hypothesis space of the embedding functions defined by the given GNN model with varying parameters. Let $\gnnModel: G_{v} \mapsto \embdingSpace$ be a specific instance of a GNN model within the hypothesis space, that maps a vertex $v \in \mathcal{V}$ to a representation vector $z_v$ in the embedding space $\embdingSpace$. Similarly, let $\hypothesis_{f}$ denote the hypothesis space of prediction functions, and $f: \embdingSpace \mapsto y_v$ be a specific instance of a prediction function that maps an embedding vector $z_v$ to a label $y_v$ in the label set. Then, an NGIL framework with GNN as the backbone and the aforementioned principle is equivalent to the following learning objective:
\begin{equation}\label{eq:learning_objective}
\begin{split}
      &  \min_{g \in \hypothesis_{g}, f \in \hypothesis_{f}} 
 \underbrace{\loss(f(g(G_{\vertexSet_i})), y_{\vertexSet_i})}_{\text{Training Loss}} 
    + \\
    & \underbrace{\alpha\cdot\sum_{j=1}^{i-1} d_{\mathrm{MMD}}(\prob(g(\mathbf{G}_{\vertexSet_j})|\graphStruct_{\task_{i-1}}),\prob(g(\mathbf{G}_{\vertexSet_j})|\graphStruct_{\task_{i}}))}_{\text{Structural Shift Mitigation}}  \\
    & + \underbrace{ \beta\cdot d_{\mathrm{MMD}}(\prob(g(\mathbf{G}_{\vertexSet_j})|\graphStruct_{\task_{i-1}}),\prob(g(\mathbf{G}_{\vertexSet_i})|\graphStruct_{\task_{i}}))}_{\text{Structural Shift Mitigation}}
\end{split}
\end{equation}
where $\loss(f(g(G_{\vertexSet_i})), y_{\vertexSet_i})$ is the training loss for measuring the prediction performance of the given model on the newest task, and $\alpha, \beta$ are the hyper-parameter that control the regularization effect.  The second and third term captures the distance between the representations of vertices in the previous and current graph structures and serves to encourage the model to converge to a latent space that is invariant to the distributional shift caused by the evolving graph structure. We term ~\eqref{eq:learning_objective} as structural shift risk mitigation, SSRM.

 \begin{theorem}[Induced CFR Bound]\label{thm:induced_cfr_bound}
 Let $g$ be a given GNN model that maps vertices to the embedding space. Then for a given loss function $\loss^q(a,b)$ of the form $|a - b|^q$ and every prediction function $f \in \hypothesis_f$, we have that
\begin{equation*}
\begin{split}
        \cataForget & (f|g)  \leq R_{\jointDistCondVerGraphInduced{2}{2}}^{\loss}(f|g)  \\
    &+ 2*d_{\mathrm{MMD}}(\prob(g(\rand{G}_{v})|\vertexSet_1, 
  \graphStruct_{\task_1}), \prob(g(\rand{G}_{v})|\vertexSet_1, 
 \graphStruct_{\task_2}))\\
  &+  d_{\mathrm{MMD}}(\prob(g(\rand{G}_{v})|\vertexSet_1, 
  \graphStruct_{\task_1}), \prob(g(\rand{G}_{v})|\vertexSet_2, 
 \graphStruct_{\task_2}))
    + \lambda',
\end{split}
\end{equation*}
where $\cataForget(f|g)$ is the catastrophic forgetting bound given a fixed GNN $g$ and
$$\lambda' = \min_{f \in \hypothesis_f}   R_{\jointDistCondVerGraphInduced{1}{2}}^{\loss^q}(h)+R_{\jointDistCondVerGraphInduced{2}{2}}^{\loss^q}(f) $$
\end{theorem}
The proof of Theorem~\ref{thm:induced_cfr_bound} can be found in Appendix~\ref{appendix:theorem2_proof}. Theorem~\ref{thm:induced_cfr_bound} can be interpreted in a similar way as Theorem~\ref{thm:forget_bound}. However, instead of operating on the input space, Theorem~\ref{thm:induced_cfr_bound} operates on the latent space learned by the GNN model. The theorem suggests that by encouraging the GNN model to converge to a latent space that reduces the distance induced by structural shift, we can decrease the risk of catastrophic forgetting of the prediction function. This validates our proposed method as given in ~\eqref{eq:learning_objective}.

In practice, we can estimate MMD by comparing the square distance between the empirical kernel mean embedding, as shown in ~\eqref{eq:emb_mmd}.
\begin{equation}\label{eq:emb_mmd}
    \begin{split}
        & \widehat{d}_{\mathrm{MMD}}^2 (X,Y) = \frac{1}{n_1^2} \sum_{i}^{n_1} \sum_{j}^{n_1} \kernel(x_i,x_j) \\
        & +\frac{1}{n_2^2} \sum_{i}^{n_2} \sum_{j}^{n_2} \kernel(y_i,y_j) - \frac{2}{n_1 n_2} \sum_{i}^{n_2} \sum_{j}^{n_1} \kernel(x_j,y_i),
    \end{split}
\end{equation}
 where $n_1,n_2$ represent the number of samples from the two distributions $X,Y$, respectively, and $\kernel$ is the chosen kernel function. In our experiment, we use the Gaussian kernel~\citep{dziugaite2015training} with $\kernel(x,y) = \sum_{\alpha_i} e^{-\alpha_i ||x-y||_2}$ ($\alpha_i = 1, 0.1, 0.01$). The overall procedure and graphical illustration of our proposed method can be found in Appendix~\ref{appendix: procedure}.

\section{Experiment}\label{sec:evaluation}
In this section, we present the experimental results of our proposed method, SSRM, on several NGIL benchmark datasets. We first describe the datasets and experimental set-up, followed by the results and analysis of our proposed method when being applied to the state-of-the-art incremental learning methods in inductive NGIL. Due to space limitations, we provide a more comprehensive description of the datasets, experiment set-up and additional results in Appendix~\ref{appendix:exp_details} and~\ref{appendix:additional_exp_result}. 
\begin{table}[h!]
  \centering
     \vspace{-3mm}
  \caption{Incremental learning settings for each dataset.
  }  
  {\small
   \setlength\tabcolsep{4pt}
    \begin{tabular}{c|ccc}
    \toprule
        Datasets  &        OGB-Arxiv  & Reddit & CoraFull          \\
            
    \midrule
           \# vertices & 169,343 & 227,853 & 19,793\\
           \# edges   & 1,166,243 & 114,615,892 & 130,622 \\
           \# class & 40 & 40 & 70\\
    \midrule
         \# tasks & 20 & 20 & 35\\
         \# vertices / \# task & 8,467 & 11,393 & 660\\
         \# edges / \# task & 58,312 & 5,730,794 & 4,354\\
    \bottomrule
    \end{tabular}%
   }
   \vspace{-3mm}
  \label{tab:data_description}%
\end{table}%

\begin{figure}[t!]
\subfigure[Evolution of Task 1 Performance]{
\includegraphics[width=.46\textwidth]{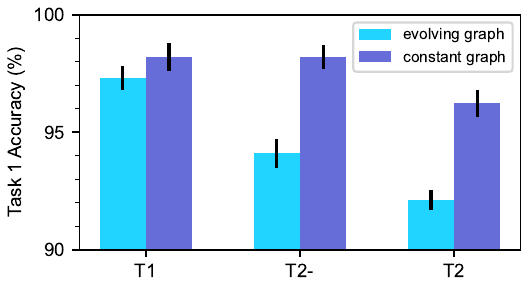}
\label{fig:structural_dependency}
}
\subfigure[Arxiv, Inductive]{
\includegraphics[width=.24\textwidth]{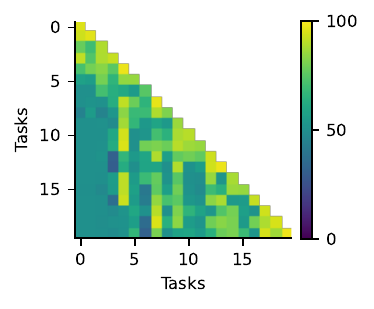}
}\hspace{-4mm}
\subfigure[Arxiv, Transductive]{
\includegraphics[width=.24\textwidth]{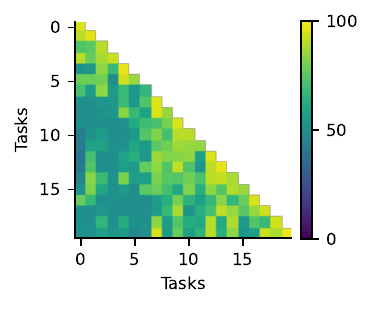}
}
 \caption{Learning Dynamics of Bare Model on Arxiv in Transductive and Inductive Settings. (a) captures the change of model performance of task 1 when transitioning into task 2 in the inductive and transductive settings. (b) and (c) are the complete performance matrix ($x,y$-axis are the $i,j$ in $r_{i,j}$ correspondingly) of Bare model in the inductive and transductive settings.}
\label{fig:diff}
\end{figure}

\paragraph{Datasets and Experimental Set-up.} We evaluate our proposed method, SSRM, on OGB-Arxiv~\citep{ogb}, Reddit~\citep{reddit}, and CoraFull~\citep{cora_full}. The experimental set-up follows the widely adopted task-incremental-learning (task-IL)~\cite{zhang2022cglb}, where a k-class classification task is extracted from the dataset for each training session. For example, in OGB-Arxiv dataset, which has 40 classes, we divide them into 20 tasks: Task 1 is a 2-class classification task between classes 0 and 1, task 2 is between classes 2 and 3, and so on. In each task, the system only has access to the graph induced by the vertices at the current and earlier learning stages, following the formulation in Sec.~\ref{sec:problem_formulation}. A brief description of the datasets, and how they are divided into different node classification tasks is given in Table~\ref{tab:data_description}. We adopt the implementation  from the GIL recent benchmark~\cite{zhang2022cglb} for creating the task-IL setting and closely follow their set-up (such as the train/valid/test set split).

% \footnote{The implementation is available at https://github.com/littleTown93/NGIL\_Evolve}

\paragraph{ Incremental Learning Frameworks.}
In our experiments, we evaluate the effectiveness of our proposed method on two state-of-art NGIL frameworks, Experience Replay GNN (ER-GNN)~\cite{liu2021overcoming} and Topology-ware Weight Preserving (TWP)~\cite{zhou2021overcoming}. ER-GNN stores a small set of representative nodes and replays them with new tasks, while TWP penalizes the update of model parameters to preserve the topological information of previous graphs. Additionally, we also include Gradient Episodic Memory (GEM) as a state-of-the-art incremental learning framework from the DNN field. GEM operates in a similar fashion as ER-GNN, but with a different strategy for selecting representative data. In addition to these incremental learning frameworks, we also include two natural baselines for NGIL: the Bare model and Joint Training. The Bare model denotes the backbone GNN without any continual learning techniques, serving as the lower bound on continual learning performance. On the other hand, Joint Training trains the backbone GNN on all tasks simultaneously, resulting in no forgetting problems and serving as the upper bound for continual learning performance.

\paragraph{ Evaluation Metric.}
Let $r_{i,j}$ denote the performance (e.g., accuracy) on task $j$ after the model has been trained over a sequence of tasks from $1$ to $i$. Then, the forgetting of task $j$ after being trained over a sequence of tasks from $1$ to $i$ is measured by $r_{i,j}-r_{j,j}$. To better understand the dynamics of the overall performance while learning about the task sequence, we are interested in the average performance sequence (APS) $:= \{\frac{\sum_{j}^ir_{i,j}}{i}|i=1,...,m\}$ and the average forgetting sequence (AFS) $:= \{\frac{\sum_{j}^ir_{i,j}-r_{j,j}}{i}|i=1,...,m-1\}$. We use the final average performance (FAP) $:= \frac{\sum_{j}^m r_{m,j}}{m}$ and the final average forgetting (FAF) $:= \frac{\sum_{j}^m r_{m,j}-r_{j,j}}{m}$ to measure the overall effectiveness of an NGIL framework.

\begin{figure}[!t]
\vspace{-1mm}
\centering
\subfigure[Arxiv, ER-GNN ]{
\includegraphics[width=.235\textwidth]{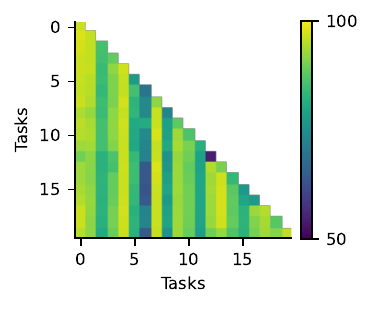}
}\hspace{-4mm}
\subfigure[Arxiv, ER-GNN-SSRM]{
\includegraphics[width=.235\textwidth]{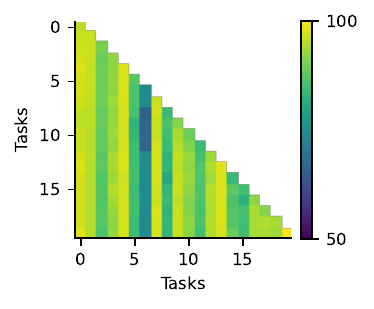}
}
\subfigure[Arxiv, Learning Curve]{
\includegraphics[width=.46\textwidth]{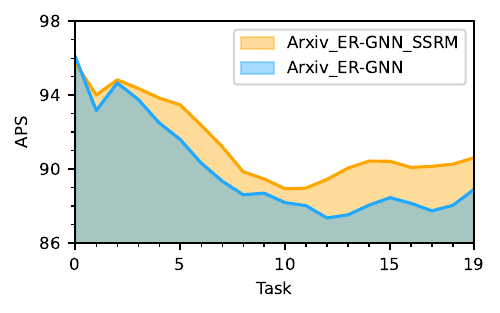}
\label{fig:learning_dynamic}
}
\caption{Learning Dynamic of ER-GNN on Arxiv w/w.o.~SSRM. (a) and (b) are the complete performance matrix ($x,y$-axis are $i,j$ in $r_{i,j}$ correspondingly) of ER-GNN on Arxiv w./w.o. SSRM. (c) is the learning curve of the two settings illustrating that SSRM leads to a higher APS for each task. }
\vspace{-2mm}
\label{fig:dynamic}
\end{figure}

\subsection{Results}
\begin{table*}[h!]
  \centering
  \caption{Performance comparison of existing NGIL frameworks w/w.o. SSRM ($\uparrow$ higher means better). Results are averaged among five trials. We use $\alpha=0.1,\beta=0.5$ for SSRM. {\bf Bold letter} with * indicates that the entry admits an improvement with SSRM.
  }  
  {
   \setlength\tabcolsep{4pt}
    \begin{tabular}{|c|cc|cc|cc|}
    \hline
    
        Dataset  &        \multicolumn{2}{c}{Arixv-CL}         & \multicolumn{2}{c}{CoraFull-CL} & \multicolumn{2}{c}{Reddit-CL} \\
    \hline
        Performance Metric& FAP (\%) $\uparrow$ & FAF (\%) $\uparrow$ & FAP (\%) $\uparrow$ & FAF (\%) $\uparrow$ & FAP(\%) $\uparrow$ & FAF (\%) $\uparrow$\\
    \hline
    Bare model & 55.9 $\pm$ 1.2 & -33.4 $\pm$ 2.3 & 58.2 $\pm$ 3.6 & -33.7  $\pm$ 3.3 & 68.6 $\pm$ 4.8 & -23.9 $\pm$ 5.7 \\
    Joint Training & 92.5 $\pm$ 0.6 & N.A.  & 94.4 $\pm$ 0.4 & N.A.   & 98.3 $\pm$ 1.2 & N.A. \\
    \hline
       GEM &   76.6 $\pm$ 1.3 &  -4.1 $\pm$ 1.4&  88.6 $\pm$ 1.1 &  -3.8 $\pm$ 0.7   &  78.8 $\pm$ 7.5 &  -17.7 $\pm$  5.6\\
       GEM-SSRM &   {\bf 80.4}* $\pm$ 1.7 &  {\bf -3.1}* $\pm$ 1.1&   {\bf 91.8}* $\pm$ 1.2 &  -3.8 $\pm$ 0.8   &  {\bf 81.5}* $\pm$ 1.5 &   {\bf -15.2}* $\pm$  1.4\\
    \hline
        ER-GNN & 86.5 $\pm$ 0.5 & -10.7 $\pm$ 0.6  & 93.7 $\pm$ 1.0 & -3.8 $\pm$ 0.4  & 95.1 $\pm$ 3.3 & -1.9. $\pm$ 0.3 \\
        ER-GNN-SSRM & {\bf 91.2}* $\pm$ 0.7 & {\bf -8.7}* $\pm$ 0.7 & {\bf 94.3}* $\pm$ 0.8 & -4.0 $\pm$ 1.9  & {\bf 97.5}* $\pm$ 0.4 & {\bf -1.8}* $\pm$ 0.2\\
   \hline
        TWP  & 86.6 $\pm$ 0.9 & -5.6 $\pm$ 0.8 & 88.1 $\pm$ 0.9 & -4.2  $\pm$ 0.9 & 89.3 $\pm$ 1.2 & -8.2 $\pm$ 1.3\\
        TWP-SSRM & {\bf 88.5}* $\pm$ 0.8 & -7.2 $\pm$0.9 & {\bf 90.2}* $\pm$ 0.7  & {\bf -3.5}* $\pm$ 0.5 & {\bf 92.5}* $\pm$ 1.7 & {\bf -7.1}* $\pm$ 1.6\\
    \hline
    \end{tabular}%
   }
   % \vspace{-2mm}
  \label{tab:improvement}%
\end{table*}%
\paragraph{ Difference in Transductive and Inductive.}
We first show the difference in task-IL between a transductive setting and an inductive setting. We do so by comparing the learning dynamic (change of $r_{i,j}$ for different $i,j$) of the Bare model. The result of the Arxiv dataset is illustrated in Fig.~\ref{fig:diff}. In Fig.~\ref{fig:structural_dependency}, $T_1$ and $T_2$ on the x-axis denote the first and the second tasks (training sessions), respectively, and $T_2-$ represents the state when Task 2 has arrived but the model has not been trained on Task 2 data yet. As shown in the figure, in the case of an evolving graph (inductive), the model performance on Task 1 experiences two drops when entering stage 2: when vertices of Task 2 are added to the graph (structural shift) and when the model is trained on Task 2's data. The bottom half of Fig.~\ref{fig:diff} is the complete performance matrix (with $r_{i,j}$ as entries) for the two settings. The colour density difference of the two performance matrices illustrates that the existing tasks suffer more catastrophic forgetting in the inductive setting as more tasks arrive. 

\begin{figure}[!t]
\vspace{-2mm}
\centering
\subfigure[Varying $\alpha$ with $\beta=0$]{
\includegraphics[width=.23\textwidth]{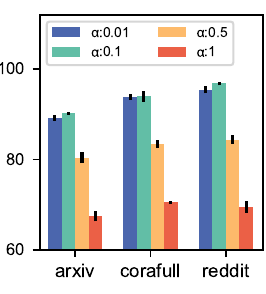}
}\hspace{-2.5mm}
\subfigure[Varying $\beta$ with $\alpha=0$]{
\includegraphics[width=.23 \textwidth]{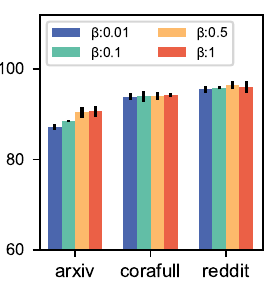}
}
% \vspace{-2mm}
\caption{Parameter sensitivity of SSRM. x-axis are the different datasets and the y-axis is FAP. The results are average of five trials. (a) captures the model performance (FAP) when varying $\alpha$ with $\beta = 0$. (b) captures the model performance (FAP) when varying $\beta$ with $\alpha = 0$. $\alpha$ and $\beta$ are the two hyperparameters used in ~\eqref{eq:learning_objective} to control the regularization effect.}
\vspace{-3mm}
\label{fig:parameter}
\end{figure}
 \paragraph{ Effectiveness of SSRM.} In Table ~\ref{tab:improvement}, we show the FAP and FAF of each method after learning the entire task sequence. On average, the Bare model without any continual learning technique performs the worst, and Joint training performs the best. FAF is inapplicable to joint-trained models because they do not follow the continual learning setting and are simultaneously trained on all tasks. In terms of FAP, methods with SSRM consistently admit an improvement with a significant margin. On the other hand, some methods do experience a slight FAF drop after applying SSRM. The reason is that SSRM simultaneously allows for a better forward transfer (transferring information from existing tasks to new tasks) and hence leads to higher performance for each new task (i.e., larger $r_{i,i}$ for each task i), as illustrated in the performance matrices and learning curve in Fig.~\ref{fig:dynamic}.

\subsection{Ablation Study}
In SSRM, there are two main hyperparameters: $\alpha$ and $\beta$ (see \eqref{eq:learning_objective} for more details). In Fig.~\ref{fig:parameter}, the performance of SSRM with different $\alpha$ and $\beta$ in $[0,1]$ is reported. The results show that the model performance is blunt to the change of $\beta$ when $\alpha = 0$. On the other hand, if $\alpha$ is too large, the representations learnt by the GNN lose expressive power over different tasks and hence the performance would decrease.

\section{Conclusion and Discussion}
This paper presents a mathematical formulation of the inductive NGIL problem and provides a comprehensive analysis of the catastrophic forgetting problem in this setting. Based on the analysis, we propose a novel method, SSRM, that addresses the problem of the structural shift in inductive NGIL by minimizing the divergence between the representations of vertices in different tasks and graphs. Our method is simple to implement and can be integrated with existing incremental learning frameworks. Empirical evaluations on benchmark datasets for NGIL show that our method significantly improves the performance of state-of-the-art NGIL frameworks in the inductive setting. 

This research opens up new opportunities for addressing the catastrophic forgetting problem in NGIL and we hope it lays down a solid foundation for further %rigorous 
research in the area. The current formulation assumes a clear task boundary. In future work, one working direction is to extend our analysis and apply SSRM to other GIL settings such as link prediction in dynamic graphs where there is no clear task boundary.

\section*{Acknowledgement}
We would like to thank the anonymous reviewers and area chairs for their helpful comments.
This work was supported in part by the following grants: 1) Hong Kong RGC under the contracts HKU 17207621 and 17203522, 2) National Natural Science Foundation of China (62206203) and 3) Fundamental Research Funds for the Central Universities (2042022kf1034).
% \newpage
\bibliography{example_paper}
\bibliographystyle{icml2023}

%%%%%%%%%%%%%%%%%%%%%%%%%%%%%%%%%%%%%%%%%%%%%%%%%%%%%%%%%%%%%%%%%%%%%%%%%%%%%%%
%%%%%%%%%%%%%%%%%%%%%%%%%%%%%%%%%%%%%%%%%%%%%%%%%%%%%%%%%%%%%%%%%%%%%%%%%%%%%%%
% APPENDIX
%%%%%%%%%%%%%%%%%%%%%%%%%%%%%%%%%%%%%%%%%%%%%%%%%%%%%%%%%%%%%%%%%%%%%%%%%%%%%%%
%%%%%%%%%%%%%%%%%%%%%%%%%%%%%%%%%%%%%%%%%%%%%%%%%%%%%%%%%%%%%%%%%%%%%%%%%%%%%%%
\newpage
\appendix
\onecolumn

\section{Appendix: Proof for Proposition~\ref{prop:imbalanced_observation}}\label{appendix:prop_proof}
\begin{proof}
First recall that we are considering an incremental learning setting consisting of two training tasks $\task_1, \task_2$ and associated observed vertex batches $\vertexSet_1,\vertexSet_2$. We define $C_{1}(\vertexSet)\mapsto \mathds{N}$ to be a function that counts the number of vertices in community 1 for a given vertex batch, and similarly defined for $C_{2}(\vertexSet) \mapsto \mathds{N}$. Consider a mean aggregation function that averages the node features of the 1-hop neighbors, i.e., $\text{mean-agg}(v) = \frac{1}{|\neighbor{1}(v)|} \sum_{u \in \neighbor{1}(v)} x_u $.

Under this setup, we have the expected input of vertex of community 1 in task 1 as follows.
\begin{equation}
\expect[\text{mean-agg}(v)|\graphStruct_{\task_1}] = \frac{C_1(V_1)p_{in}}{C_1(V_1)p_{in} + p_{out} C_2(V_1)} \mu_1 + \frac{C_2(V_1)p_{out}}{C_1(V_1)p_{in} + p_{out} C_2(V_1)} \mu_2
\end{equation}
Similarly, we have the expected input of the vertex of community 1 in task 2 as follows.
\begin{equation}
\expect[\text{mean-agg}(v)|\graphStruct_{\task_2}] = \frac{C_1(V)p_{in}}{C_1(V)p_{in} + p_{out} C_2(V)} \mu_1 + \frac{C_2(V)p_{out}}{C_1(V)p_{in} + p_{out} C_2(V)} \mu_2
\end{equation}
where $V = V_1 + V_2$.  For simplicity, let's denote $ a= C_1(V_1)p_{in}, b = p_{out} C_2(V), c = C_1(V_2)p_{in}, d = C_2(V_2)p_{out}$. 

\begin{equation}
\begin{split}
    &\expect[\text{mean-agg}(v)|\graphStruct_{\task_1}] \neq \expect[\text{mean-agg}(v)|\graphStruct_{\task_2}]\\
    &\Leftrightarrow \frac{a}{a+b} \neq \frac{a+c}{a+b+c+d}\\
    &\Leftrightarrow \frac{a}{b} \neq \frac{c}{d}
\end{split}
\end{equation}
Substitute back $ a= C_1(V_1)p_{in}, b = p_{out} C_2(V), c = C_1(V_2)p_{in}, d = C_2(V_2)p_{out}$, we obtain the result of the proposition
\end{proof}

\section{Appendix: Proof for Theorem~\ref{thm:forget_bound}}\label{appendix:theorem1_proof}
In this section, we provide proof for Theorem~\ref{thm:forget_bound}.

We use the following Lemma 35 from~\cite{da_survey}.
\begin{lemma}\label{lemma:mmd_bound}
Let $\hypothesis = \{ f \in \mathcal{H}_k : \Hnorm{f} \leq 1\}$ where $\mathcal{H}_k$ is a RKHS with its associated kernel $k$.
Let $\prob_1, \prob_2$ be two arbitrary distributions that share the same input and output space as $\mathcal{H}_k$. Then for a given loss function $\loss^q(a,b)$ of the form $|a - b|^q$ , for every $h,h' \in \mathcal{H}$, we have
$$R_{\prob_1}^{\loss^q}(h,h') \leq R_{\prob_2}^{\loss^q}(h,h') + d_{\mathrm{MMD}}(\prob_{X1},\prob_{X2}),$$
where $\prob_{X1},\prob_{X2}$ are the input distribution of $\prob_1, \prob_2$ and $R_{\prob}^{\loss^q}(h,h') = \expect_{x \sim \prob_{X}}[\loss^q(h(x),h'(x))]$. 
\end{lemma}

\begin{proof}
Let $\hypothesis = \{ f \in \mathcal{H}_k : \Hnorm{f} \leq 1\}$ where $\mathcal{H}_k$ is a RKHS with its associated kernel $k$.
Let $ \jointDistCondVerGraph{1}{1}$, $ \jointDistCondVerGraph{1}{2}$ and $ \jointDistCondVerGraph{2}{2}$ be the three distributions that characterize a NGIL-2 problem. Let $\loss^q(a,b)$ be a loss function as given in the premise of the theorem. For simplicity, we denote $\prob_1 = \jointDistCondVerGraph{1}{1}, \prob_2 = \jointDistCondVerGraph{1}{2}, \prob = \jointDistCondVerGraph{2}{2}$.

Let's denote the optimal function in the hypothesis space as,
% $$h^*_1 = \argmin_{h \in \hypothesis}  R_{\prob_1}^{\loss^q}(h)$$
$$h^*_2 = \argmin_{h \in \hypothesis}  R_{\prob_2}^{\loss^q}(h)$$
% $$h^*_3 = \argmin_{h \in \hypothesis}  R_{\prob_3}^{\loss^q}(h)$$
$$\lambda = \min_{h \in \hypothesis}  R_{\prob_3}^{\loss^q}(h)+ R_{\prob_2}^{\loss^q}(h)$$

Then, we have that $\forall h \in \hypothesis$,
\begin{equation}
\begin{split}
     R_{\prob_2}^{\loss^q}(h) & \leq R_{\prob_2}^{\loss^q}(h^*_2) + R_{\prob_2}^{\loss^q}(h,h^*_2)\\
     & =  R_{\prob_2}^{\loss^q}(h^*_2) + R_{\prob_2}^{\loss^q}(h,h^*_2) + R_{\prob_1}^{\loss^q}(h,h^*_2) - R_{\prob_1}^{\loss^q}(h,h^*_2) + R_{\prob_3}^{\loss^q}(h,h^*_2) - R_{\prob_3}^{\loss^q}(h,h^*_2) +  R_{\prob_2}^{\loss^q}(h,h^*_2) -  R_{\prob_2}^{\loss^q}(h,h^*_2) \\
     &=  R_{\prob_2}^{\loss^q}(h^*_2) + [R_{\prob_2}^{\loss^q}(h,h^*_2) - R_{\prob_1}^{\loss^q}(h,h^*_2)] + [R_{\prob_2}^{\loss^q}(h,h^*_2) - R_{\prob_3}^{\loss^q}(h,h^*_2)] + R_{\prob_1}^{\loss^q}(h,h^*_2)  + R_{\prob_3}^{\loss^q}(h,h^*_2)  - R_{\prob_2}^{\loss^q}(h,h^*_2)
\end{split}
\end{equation}

Applying Lemma~\ref{lemma:mmd_bound} on $\prob_1, \prob_2$ we have that 
\begin{equation}
\begin{split}
        & R_{\prob_2}^{\loss^q}(h,h') \leq R_{\prob_1}^{\loss^q}(h,h') + d_{\mathrm{MMD}}(\prob_{X1},\prob_{X2}) \\
        & \Leftrightarrow R_{\prob_2}^{\loss^q}(h,h') - R_{\prob_1}^{\loss^q}(h,h') \leq d_{\mathrm{MMD}}(\prob_{X1},\prob_{X2})
\end{split}
\end{equation}
Similarly, applying Lemma~\ref{lemma:mmd_bound} on $\prob_2, \prob_3$ we have that 
\begin{equation}
  R_{\prob_2}^{\loss^q}(h,h') - R_{\prob_3}^{\loss^q}(h,h') \leq d_{\mathrm{MMD}}(\prob_{X2},\prob_{X3})
\end{equation}
Substitute back to the inequality above, we have that
\begin{equation}
\begin{split}
     R_{\prob_2}^{\loss^q}(h) & \leq R_{\prob_2}^{\loss^q}(h^*_2) + R_{\prob_2}^{\loss^q}(h,h^*_2)\\
     &\leq  R_{\prob_2}^{\loss^q}(h^*_2) + [R_{\prob_2}^{\loss^q}(h,h^*_2) - R_{\prob_1}^{\loss^q}(h,h^*_2)] + [R_{\prob_2}^{\loss^q}(h,h^*_2) - R_{\prob_3}^{\loss^q}(h,h^*_2)] + R_{\prob_1}^{\loss^q}(h,h^*_2)  + R_{\prob_3}^{\loss^q}(h,h^*_2)  - R_{\prob_2}^{\loss^q}(h,h^*_2)\\
     &\leq R_{\prob_2}^{\loss^q}(h^*_2) + 2d_{\mathrm{MMD}}(\prob_{X1},\prob_{X2}) + d_{\mathrm{MMD}}(\prob_{X2},\prob_{X3}) +  R_{\prob_3}^{\loss^q}(h,h^*_2)  \\
     &\leq R_{\prob_2}^{\loss^q}(h^*_2) + 2d_{\mathrm{MMD}}(\prob_{X1},\prob_{X2}) + d_{\mathrm{MMD}}(\prob_{X2},\prob_{X3}) +  R_{\prob_3}^{\loss^q}(h) + R_{\prob_3}^{\loss^q}(h^*_2) \\
     & = R_{\prob_3}^{\loss^q}(h)  + 2d_{\mathrm{MMD}}(\prob_{X1},\prob_{X2}) + d_{\mathrm{MMD}}(\prob_{X2},\prob_{X3}) +  R_{\prob_3}^{\loss^q}(h) + R_{\prob_3}^{\loss^q}(h^*_2) +R_{\prob_2}^{\loss^q}(h^*_2) \\
     &\leq R_{\prob_3}^{\loss^q}(h)  + 2d_{\mathrm{MMD}}(\prob_{X1},\prob_{X2}) + d_{\mathrm{MMD}}(\prob_{X2},\prob_{X3}) +  R_{\prob_3}^{\loss^q}(h) + R_{\prob_3}^{\loss^q}(h^*_2) + \lambda\\
\end{split}
\end{equation}
Substitute $\prob_1 = \jointDistCondVerGraph{1}{1}, \prob_2 = \jointDistCondVerGraph{1}{2}, \prob = \jointDistCondVerGraph{2}{2}$ back to the equation above, we got 

\begin{equation}
\begin{split}
            \cataForget(h) & \leq R_{\jointDistCondVerGraph{2}{2}}^{\loss^q}(h)  
    +  2d_{\mathrm{MMD}}(\prob(\rand{G}_{v}|\vertexSet_1, 
  \graphStruct_{\task_1}), \prob(\rand{G}_{v}|\vertexSet_1, 
 \graphStruct_{\task_2}))
  +  d_{\mathrm{MMD}}(\prob(\rand{G}_{v}|\vertexSet_1, 
  \graphStruct_{\task_1}), \prob(\rand{G}_{v}|\vertexSet_2, 
 \graphStruct_{\task_2}))
    + \lambda.
\end{split}
\end{equation}

\end{proof}
\section{Proof of Theorem~\ref{thm:induced_cfr_bound}}\label{appendix:theorem2_proof}
In this section, we provide the proof for Theorem~\ref{thm:induced_cfr_bound}.

\begin{proof}
Let $\hypothesis_f$ be the hypothesis space for the prediction head and $g$ be a given GNN that maps a vertex into an embedding space $\mathcal{Z}$. Then, a given $g$ induces a distribution in the latent space $\mathcal{Z}$ for the prediction function $f \in \hypothesis_f$. For simplicity, let's denote the induced distribution as, $\prob_1 = R_{\jointDistCondVerGraphInduced{1}{1}}, \prob_2 = \jointDistCondVerGraphInduced{1}{2}, \prob_3 = \jointDistCondVerGraphInduced{2}{2}$.

Let's denote the optimal function in the hypothesis space as,

$$f^*_2 = \argmin_{h \in \hypothesis_f}  R_{\prob_2}^{\loss^q}(f)$$
$$\lambda' = \min_{f \in \hypothesis_f}  R_{\prob_3}^{\loss^q}(f)+ R_{\prob_2}^{\loss^q}(f)$$

Then, we have that $\forall f \in \hypothesis_f$,
\begin{equation}
\begin{split}
     R_{\prob_2}^{\loss^q}(f) & \leq R_{\prob_2}^{\loss^q}(f^*_2) + R_{\prob_2}^{\loss^q}(f,f^*_2)\\
     & =  R_{\prob_2}^{\loss^q}(f^*_2) + R_{\prob_2}^{\loss^q}(f,f^*_2) + R_{\prob_1}^{\loss^q}(f,f^*_2) - R_{\prob_1}^{\loss^q}(f,f^*_2) + R_{\prob_3}^{\loss^q}(f,f^*_2) - R_{\prob_3}^{\loss^q}(f,f^*_2) +  R_{\prob_2}^{\loss^q}(f,f^*_2) -  R_{\prob_2}^{\loss^q}(f,f^*_2) \\
     &=  R_{\prob_2}^{\loss^q}(h^*_2) + [R_{\prob_2}^{\loss^q}(f,f^*_2) - R_{\prob_1}^{\loss^q}(f,f^*_2)] + [R_{\prob_2}^{\loss^q}(f,f^*_2) - R_{\prob_3}^{\loss^q}(f,f^*_2)] + R_{\prob_1}^{\loss^q}(f,f^*_2)  + R_{\prob_3}^{\loss^q}(f,f^*_2)  - R_{\prob_2}^{\loss^q}(f,f^*_2)
\end{split}
\end{equation}

Then, following the same procedure in the proof of Theorem~\ref{thm:forget_bound}, we can get

\begin{equation}
\begin{split}
             \cataForget & (f|g)  \leq R_{\jointDistCondVerGraphInduced{2}{2}}^{\loss}(f|g)  
    +  2*d_{\text{MMD}}(\prob(g(\rand{G}_{v})|\vertexSet_1, 
  \graphStruct_{\task_1}), \prob(g(\rand{G}_{v})|\vertexSet_1, 
 \graphStruct_{\task_2})) \\
  &+  d_{\text{MMD}}(\prob(g(\rand{G}_{v})|\vertexSet_1, 
  \graphStruct_{\task_1}), \prob(g(\rand{G}_{v})|\vertexSet_2, 
 \graphStruct_{\task_2}))
    + \lambda',
\end{split}
\end{equation}

\end{proof}
\section{Appendix:Overall Procedure}\label{appendix: procedure}
The overall learning procedure in each stage is summarized in Algorithm~\ref{alg:train} and Fig.~\ref{fig:train} provide a graphical illustration of the procedure. It involves training the GNN model on the current task while also minimizing the distance between the representations of vertices in the previous and current graph structures through the inclusion of the structural shift mitigation term. This results in a model that is able to maintain good performance on previous tasks while also learning the current task effectively, leading to improved generalization performance.

\begin{algorithm}[!h]
  \caption{Learning Procedure for $\task_i$}
  \label{alg:train}
\begin{algorithmic}
  \STATE {\bfseries Input:} $V_i$  \hspace{8mm} //new coming vertex set for $\task_i$\\
  \REQUIRE $f,g$ \hspace{2mm} //current model parameter \\ 
  \REQUIRE $I_{\task_{i-1}}$ \hspace{5mm} //a set of vertices sampled form $V_1,...,V_{i-1}$ \\
  \REQUIRE $\beta$, $\alpha$ \hspace{5mm} //hyper parameter for controling the regularization effect \\
  Repeat until convergence: \\
  Compute the representation for $I$ before and after $V_i$
  $$Z_{\mathrm{bef}} = g(I_{\task_{i-1}},\theta_g|\graphStruct_{\task_{i-1}}),\quad \mathcal{Z}_{\mathrm{aft}} = g(I_{\task_{i-1}},\theta_g|\graphStruct_{\task_{i}}) $$
  $$\mathcal{Z}_{V_i} = g(V_i,\theta_g|\graphStruct_{\task_{i}})$$
   Update $f,g$ through training procedure (e.g. Back Propogation ) with loss function:\\
  $$ \loss(f(g(G_{\vertexSet_i},\theta_g),\theta_f), y_{\vertexSet_i}) + \mathrm{Reg}$$
  where
  \begin{equation}\label{eq:mmdreg}
      \mathrm{Reg} = \beta \hat{d_{\mathrm{MMD}^2}}(Z_{\mathrm{bef}}, \mathcal{Z}_{\mathrm{aft}}) + \alpha  \hat{d_{\mathrm{MMD}^2}}(Z_{\mathrm{bef}}, \mathcal{Z}_{V_i})
  \end{equation}
\end{algorithmic}
\end{algorithm}

\begin{figure}[!t]
\centering
\includegraphics[width=0.8\linewidth]{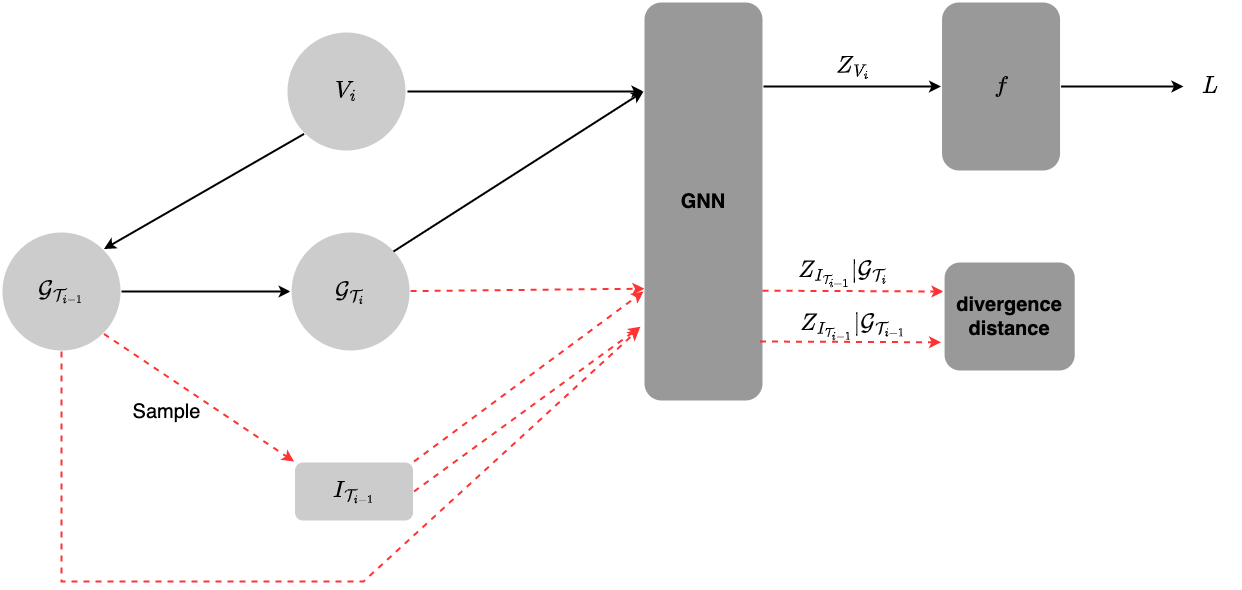}
\caption{Illustration of the structural shift risk mitigation}
\label{fig:train}
\end{figure}

\begin{figure}[!h]
\centering
\includegraphics[width=0.7\linewidth]{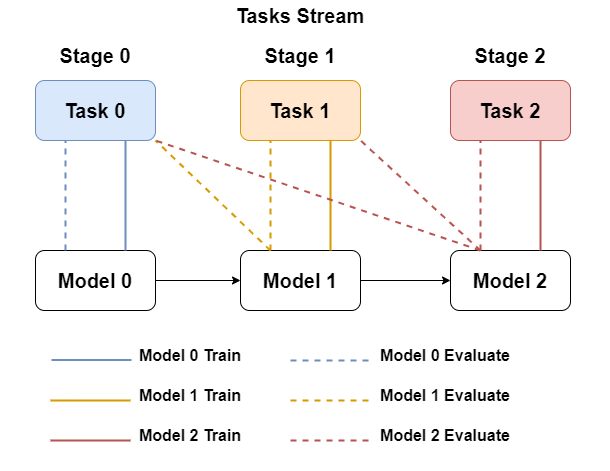}
\caption{Incremental Training in Development Phase. The figure above illustrates what task data the model used for training and evaluation.
}
\label{fig:it}
\end{figure}

\subsection{Space Complexity}
To compute the maximum mean discrepancy (MMD) given in Eq.~\eqref{eq:mmdreg} in the SSRM framework, we need the representation of $Z_{\mathrm{bef}}, \mathcal{Z}_{V_i}$ and $\mathcal{Z}_{\mathrm{aft}}$. We can obtain $\mathcal{Z}_{\mathrm{aft}}$ and $\mathcal{Z}_{V_i}$ from the updated graph at each task. The only extra storage cost incurred by the proposed method comes from $Z_{\mathrm{bef}}$, and the number of data required for $Z_{\mathrm{bef}}$ is determined by the sampling efficiency of MMD, that has been empirically shown to be efficient. Furthermore, as SSRM is not a stand-alone incremental learning framework but serves as a regularization to mitigate the effect of structural shift and boost the performance of existing incremental learning frameworks. As most existing incremental learning frameworks require storage or access to the previous data, SSRM could utilize this fact and make use of the data that have already been used by the incremental learning framework. In other words, for most incremental learning frameworks, such as ER-GNN, SSRM does not incur extra storage costs. 

\subsection{Computation Complexity}
First, recall that the equation we used for estimating the maximum mean discrepancy (MMD) is given as follows.
\begin{equation}\label{eq:emb_mmd2}
    \begin{split}
         \widehat{d}_{\mathrm{MMD}}^2 (X,Y) = \frac{1}{n_1^2} \sum_{i}^{n_1} \sum_{j}^{n_1} \kernel(x_i,x_j) 
        +\frac{1}{n_2^2} \sum_{i}^{n_2} \sum_{j}^{n_2} \kernel(y_i,y_j) - \frac{2}{n_1 n_2} \sum_{i}^{n_2} \sum_{j}^{n_1} \kernel(x_j,y_i),
    \end{split}
\end{equation}
Let $N = \max \{n_1, n_2\}$ where $n_1, n_2$ are the constants in Eq.~\ref{eq:emb_mmd2}.
It is obvious that the computation complexity of Eq.~\ref{eq:emb_mmd2} above is $\mathcal{O}(N^2)$. Note that subsampling is commonly employed for estimating MMD, i.e., $n_1$ and $n_2$ (hence $N$)  would be small and therefore, the cost of SSRM is small compared with the standard training time. Our empirical study finds that the extra computation cost from MMD is about ~5\% of the total computation time. There exists other techniques such as kernel approximation to make MMD computation even more scalable.

\section{Appendix: Additional Experiment Details}\label{appendix:exp_details}

\subsection{Hardware and Software}
All the experiments of this paper are conducted on the following machine

CPU: two Intel Xeon Gold 6230 2.1G, 20C/40T, 10.4GT/s, 27.5M Cache, Turbo, HT (125W) DDR4-2933

GPU: four NVIDIA Tesla V100 SXM2 32G GPU Accelerator for NV Link

Memory: 256GB (8 x 32GB) RDIMM, 3200MT/s, Dual Rank

OS: Ubuntu 18.04LTS

\subsection{Dataset and Processing}
\subsection{Dataset Description}
{\bf OGB-Arxiv.} The OGB-Arxiv dataset~\citep{ogb} is a benchmark dataset for node classification. It is constructed from the arXiv e-print repository, a popular platform for researchers to share their preprints. The graph structure is constructed by connecting papers that cite each other. The node features include the text of the paper's abstract, title, and its authors' names. Each node is assigned one of 40 classes, which corresponds to the paper's main subject area. 

{\bf Cora-Full.} The Cora-Full~\cite{cora_full} is a benchmark dataset for node classification. Similarly to OGB-Arxiv, it is a citation network consisting of 70 classes. 

{\bf Reddit.} The Reddit dataset~\citep{reddit} is a benchmark dataset for node classification that consists of posts and comments from the website Reddit.com. Each node represents a post or comment and each edge represents a reply relationship between posts or comments.

\begin{table}[h!]
  \centering
     \vspace{-3mm}
  \caption{Incremental learning settings for each dataset.
  }  
  {\large
   \setlength\tabcolsep{4pt}
    \begin{tabular}{c|ccc}
    \toprule
        Datasets  &        OGB-Arxiv  & Reddit & CoraFull          \\
            
    \midrule
           \# vertices & 169,343 & 227,853 & 19,793\\
           \# edges   & 1,166,243 & 114,615,892 & 130,622 \\
           \# class & 40 & 40 & 70\\
    \midrule
         \# tasks & 20 & 20 & 35\\
         \# vertices / \# task & 8,467 & 11,393 & 660\\
         \# edges / \# task & 58,312 & 5,730,794 & 4,354\\
    \bottomrule
    \end{tabular}%
   }
   \vspace{-3mm}
  \label{tab:data_description2}%
\end{table}%

\subsubsection{License}
The datasets used in this paper are curated from existing public data sources and follow their licenses. OGB-Arxiv is licensed under Open Data Commons Attribution License (ODC-BY).  Cora-Full dataset and the Reddit dataset  are two datasets built from publicly available sources (public papers and Reddit posts) without a license attached by the authors.

\subsubsection{Data Processing}
For the datasets, we remove the 41-th class of Reddit-CL, following closely in ~\cite{zhang2022cglb}. This aims to ensure an even number of classes for each dataset to be divided into a sequence of 2-class tasks. For all the datasets, the train-validation-test splitting ratios
are 60\%, 20\%, and 20\%. The train-validation-test
splitting is obtained by random sampling, therefore the performance may be slightly different with splittings from different rounds of random sampling.

\subsection{Hyperparameter of Incremental Learning Framework}
Table~\ref{tab:hyper} is the hyperparameter research space we adopt from ~\cite{zhang2022cglb}. 
\begin{table}[h!]
  \centering
  \caption{Incremental learning settings for each dataset.
  }  
  {\large
   \setlength\tabcolsep{4pt}
    \begin{tabular}{c|c}
    \hline
     GEM & memory\_strength:[0.05,0.5,5]; n\_memories:[10,100,1000]\\
     TWP & lambda\_1:[100,10000]; lambda\_t:[100,10000]; beta:[0.01,0.1] \\
     ER-GNN & budget:[10,100]; d:[0.05,0.5,5.0]; sampler:[CM]\\
     \hline
    \end{tabular}%
   }
  \label{tab:hyper}%
\end{table}%
\section{Appendix:Additional Experiment Result}\label{appendix:additional_exp_result}
In this section, we present additional experimental results.

\subsection{Difference between Inductive and Transductive}
In this subsection, we provide the additional results on the remaining dataset for the difference between inductive and transductive settings. The results are reported in Fig.~\ref{fig:addition1} and Fig.~\ref{fig:addition2}.

\begin{figure}[!h]
    \centering
    \subfigure[Reddit, Inductive]{
    \includegraphics[width=.4\textwidth]{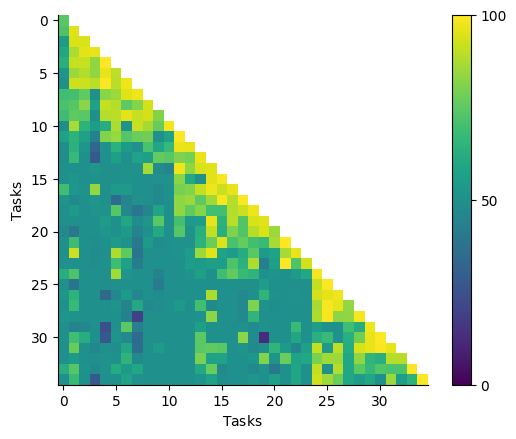}
    }
    \subfigure[Reddit, Transductive]{
    \includegraphics[width=.4\textwidth]{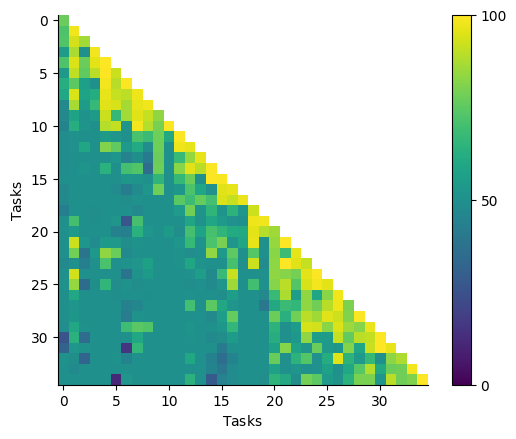}
    }
    \subfigure[CoraFull, Inductive]{
    \includegraphics[width=.4\textwidth]{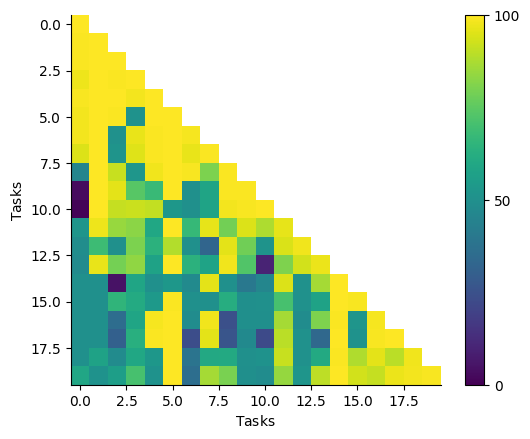}
    }
    \subfigure[CoraFull, Transductive]{
    \includegraphics[width=.4\textwidth]{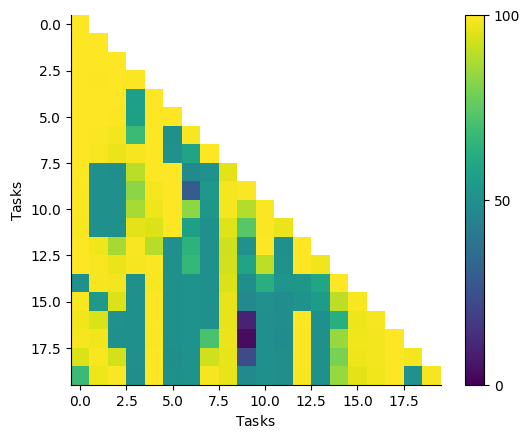}
    }
    \caption{Performance Matrix of Bare Model on Transductive and Inductive Setting on CoraFull and Reddit Datasets.}
    \label{fig:addition1}
\end{figure}

\begin{figure}[!h]
    \centering
    \subfigure[CoraFull, Inductive]{
    \includegraphics[width=.45\textwidth]{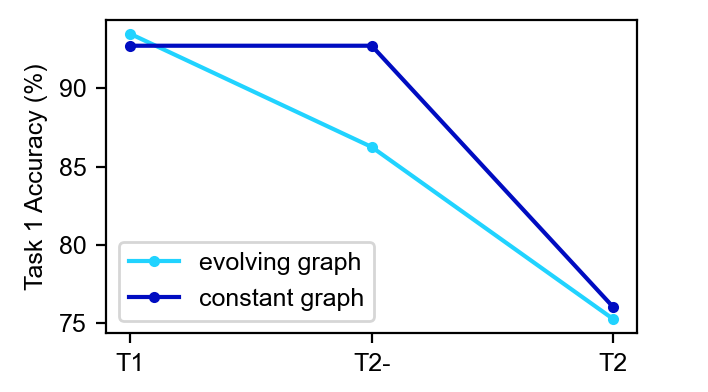}
    }
    \subfigure[Reddit, Transductive]{
    \includegraphics[width=.45\textwidth]{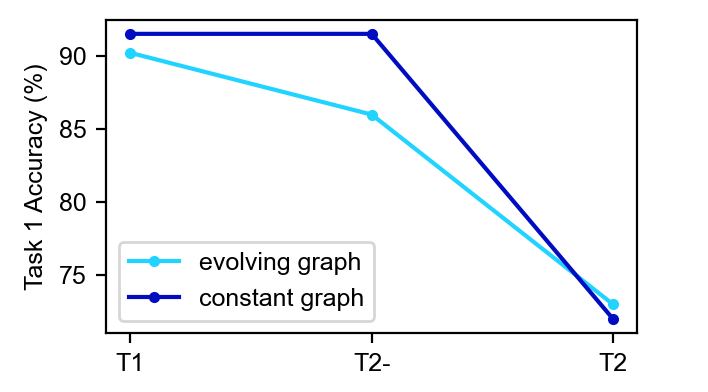}
    }
    \caption{Evolution of Task 1 performance.}
    \label{fig:addition2}
\end{figure}

\subsection{Incremental Learning w/w.o. SSRM}  
In this subsection, we provide the additional results on the remaining dataset for the effect of SSRM on the complete learning dynamic. The results are reported in Fig.~\ref{fig:addition3}.

\begin{figure}
    \centering
    \subfigure[Reddit w. SSRM]{
    \includegraphics[width=.4\textwidth]{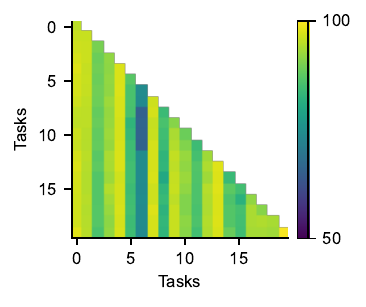}
    }
    \subfigure[Reddit w.o. SSRM]{
    \includegraphics[width=.4\textwidth]{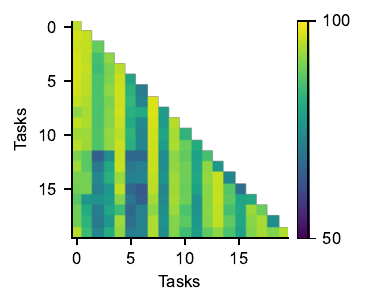}
    }
    \subfigure[CoraFull w. SSRM]{
    \includegraphics[width=.4\textwidth]{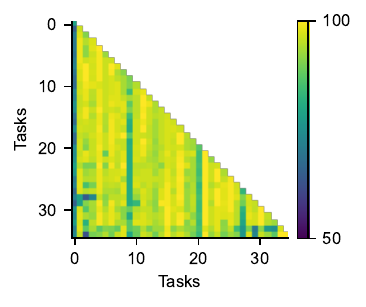}
    }
    \subfigure[CoraFull, w.o. SSRM]{
    \includegraphics[width=.4\textwidth]{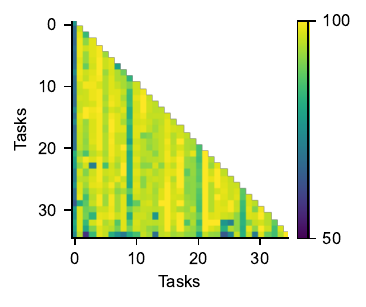}
    }
    \caption{Performance Matrix of ER-GNN with SSRM and without SSRM on CoraFull and Reddit Datasets.}
    \label{fig:addition3}
\end{figure}

\begin{figure}
    \centering
    \subfigure[Arxiv w. SSRM]{
    \includegraphics[width=.4\textwidth]{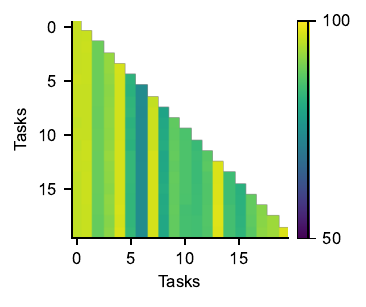}
    }
    \subfigure[Arxiv w.o. SSRM]{
    \includegraphics[width=.4\textwidth]{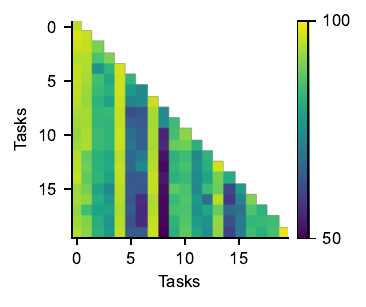}
    }
    \subfigure[Reddit w. SSRM]{
    \includegraphics[width=.4\textwidth]{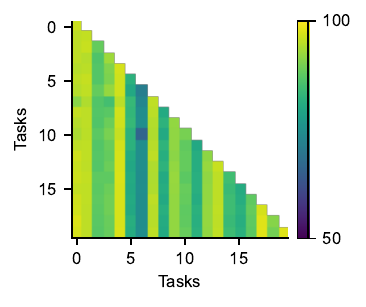}
    }
    \subfigure[Reddit w.o. SSRM]{
    \includegraphics[width=.4\textwidth]{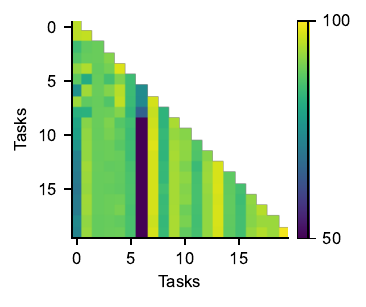}
    }
    \subfigure[CoraFull w. SSRM]{
    \includegraphics[width=.4\textwidth]{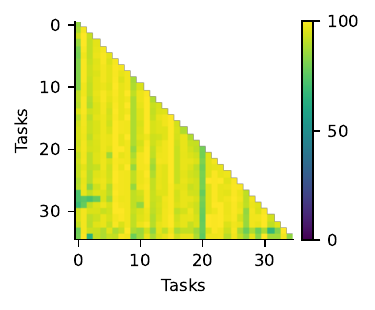}
    }
    \subfigure[CoraFull, w.o. SSRM]{
    \includegraphics[width=.4\textwidth]{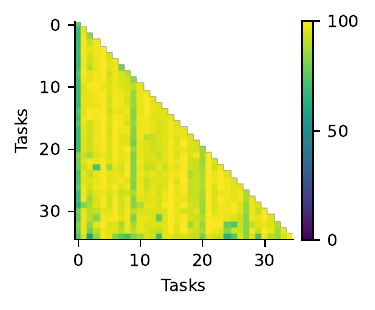}
    }
    \caption{Performance Matrix of TWP with SSRM and without SSRM on Arxiv, CoraFull and Reddit Datasets.}
    \label{fig:addition3}
\end{figure}
\section{Additional Related Work}\label{sec:related_work}
\subsection{Incremental Learning}
Incremental learning, also known as continual or lifelong learning, has gained increasing attention in recent years and has been extensively explored on Euclidean data. We refer readers to the surveys~\cite{cl_survey,cl_survey2,cl_nlp} for a more comprehensive review of these works. The primary challenge of incremental learning is to address the catastrophic forgetting problem, which refers to the drastic degradation in a model's performance on previous tasks after being trained on new tasks.

Existing approaches for addressing this problem can be broadly categorized into three types: regularization-based methods, experience-replay-based methods, and parameter-isolation-based methods. Regularization-based methods aim to maintain the model's performance on previous tasks by penalizing large changes in the model parameters~\cite{jung2016less,li2017learning,kirkpatrick2017overcoming,farajtabar2020orthogonal,saha2021gradient}. Parameter-isolation-based methods prevent drastic changes to the parameters that are important for previous tasks by continually introducing new parameters for new tasks ~\cite{rusu2016progressive,yoon2017lifelong,yoon2019scalable,wortsman2020supermasks,wu2019large}. Experience-replay-based methods select a set of representative data from previous tasks, which are used to retrain the model with the new task data to prevent forgetting ~\cite{lopez2017gradient,shin2017continual,aljundi2019gradient,caccia2020online,chrysakis2020online,knoblauch2020optimal}.

Our proposed method, SSRM, is a novel regularization-based technique that addresses the unique challenge of the structural shift in the inductive NGIL. Unlike existing regularization methods, which focus on minimizing the effect of updates from new tasks, SSRM aims to minimize the impact of the structural shift on the generalization of the model on the existing tasks by finding a latent space where the impact of the structural shift is minimized. This is achieved by minimizing the distance between the representations of vertices in the previous and current graph structures through the inclusion of a structural shift mitigation term in the learning objective. It is important to note that SSRM should not be used as a standalone method to overcome catastrophic forgetting but should be used as an additional regularizer to improve performance when there is a structural shift.

\subsection{Graph Incremental Learning}
Recently, there has been a surge of interest in GIL due to its practical significance in various applications~\cite{wang2022lifelong,xu2020graphsail,daruna2021continual,kou2020disentangle,ahrabian2021structure,cai2022multimodal,wang2020bridging,liu2021overcoming,zhang2021hierarchical,zhou2021overcoming,carta2021catastrophic,zhang2022cglb,kim2022dygrain,tan2022graph,su2023robust}. However, most existing works in NGIL focus on a transductive setting, where the entire graph structure is known before performing the task or where the inter-connection between different tasks is ignored. In contrast, the inductive NGIL problem, where the graph structure evolves as new tasks are introduced, is less explored and lacks a clear problem formulation and theoretical understanding. There is currently a gap in understanding the inductive NGIL problem, which our work aims to address. In this paper, we highlight the importance of addressing the structural shift and propose a novel regularization-based technique, SSRM, to mitigate the impact of the structural shift on the inductive NGIL problem. Our work lays down a solid foundation for future research in this area.

Finally, it is important to note that there is another area of research known as dynamic graph learning~\cite{galke2021lifelong,wang2020streaming,han2020graph,yu2018netwalk,nguyen2018continuous,ma2020streaming,feng2020incremental,bielak2022fildne,su2024pres}, which focuses on enabling GNNs to capture the changing graph structures. The goal of dynamic graph learning is to capture the temporal dynamics of the graph into the representation vectors, while having access to all previous information. In contrast, GIL addresses the problem of catastrophic forgetting, in which the model's performance on previous tasks degrades after learning new tasks. For evaluation, a dynamic graph learning algorithm is only tested on the latest data, while GIL models are also evaluated on past data. Therefore, dynamic graph learning and GIL are two independent research directions with different focuses and should be considered separately.

\end{document}